\documentclass{article}
\pdfoutput=1

\PassOptionsToPackage{numbers, compress}{natbib}

\usepackage[final]{nips_2016}

\makeatletter
\renewcommand{\@noticestring}{}
\makeatother

\usepackage[utf8]{inputenc} % allow utf-8 input
\usepackage[T1]{fontenc}    % use 8-bit T1 fonts
\usepackage{hyperref}       % hyperlinks
\usepackage{url}            % simple URL typesetting
\usepackage{booktabs}       % professional-quality tables
\usepackage{amsfonts}       % blackboard math symbols
\usepackage{nicefrac}       % compact symbols for 1/2, etc.
\usepackage{microtype}      % microtypography

\usepackage{amsmath, amssymb, amsthm}
\usepackage{algorithm,algpseudocode}
\usepackage{graphicx}

\newcommand{\except}{\setminus}

\newcommand{\E}{\mathbb E}
\newcommand{\V}{\mathbb V}
\newcommand{\Cov}{\text{Cov}}
\newcommand{\VLR}{V_{\text{LR}}}
\newcommand{\VReLEG}{V_{\text{ours}}}

\newcommand{\q}{q_\phi(z | x)}

\newcommand{\df}[2]{\partial #1 / \partial #2}
\newcommand{\Df}[2]{{\partial #1 \over \partial #2}}

\newcommand{\noise}{\epsilon}
\newcommand{\pa}{\text{pa}}
\newcommand{\qi}{q_{\phi_i}(z_i | \pa_i)}
\newcommand{\qinoise}{q_{\phi_i}(z_i | \pa_i(\noise_{\except i}))}

\newcommand{\g}{g_{\phi}(x, \noise)}
\newcommand{\gi}{g_{\phi_i}(\pa_i, \noise_i)}
\newcommand{\gip}{g_{\phi_{i'}}(\pa_{i'}, \noise_{i'})}
\newcommand{\gxi}{h_{\phi_{\except i}}(x, z_i, \noise_{\except i})}

\newcommand{\secref}[1]{Sec.\ref{sec:#1}}
\newcommand{\theoref}[1]{Theorem \ref{thm:#1}}

\newcommand{\eref}[1]{Eq. \eqref{eq:#1}}
\newcommand{\aref}[1]{Alg. \ref{alg:#1}}

\newtheorem{theorem}{Theorem}
\newtheorem{lemma}[theorem]{Lemma}

\title{Reparameterization trick for discrete variables}

\author{
    Seiya Tokui\\
        Preferred Networks, Inc.\\
        The University of Tokyo\\
        Tokyo 100-0004, Japan\\
        \texttt{tokui@preferred.jp} \\
    \And
    Issei Sato\\
        The University of Tokyo\\
        Tokyo 113-8656, Japan\\
        \texttt{sato@k.u-tokyo.ac.jp}
}

\begin{document}
% \nipsfinalcopy is no longer used

\maketitle

\begin{abstract}
  Low-variance gradient estimation is crucial for learning directed graphical models parameterized by neural networks, where the reparameterization trick is widely used for those with continuous variables.
  While this technique gives low-variance gradient estimates, it has not been directly applicable to discrete variables, the sampling of which inherently requires discontinuous operations.
  We argue that the discontinuity can be bypassed by marginalizing out the variable of interest, which results in a new reparameterization trick for discrete variables.
  This reparameterization greatly reduces the variance, which is understood by regarding the method as an application of common random numbers to the estimation.
  The resulting estimator is theoretically guaranteed to have a variance not larger than that of the likelihood-ratio method with the optimal input-dependent baseline.
  We give empirical results for variational learning of sigmoid belief networks.
\end{abstract}

\section{Introduction} \label{sec:intro}

Directed graphical models parameterized by neural networks are widely used for complicated data distributions in high dimensional spaces, which require high levels of non-linearity and uncertainty to be captured.
For learning such models, the objective function is often given as an expectation of a nonlinear function over latent variables, e.g. the evidence lower bound of deep directed generative models \citep{Kingma14,Mnih14}.
In this case, computing the exact gradient of the expectation is generally infeasible, and it has to be estimated approximately.
For each variable, we consider a problem of estimating the gradient of the objective function w.r.t. parameters on which the variable directly depends.

When the variables are modeled by certian continuous distributions such as a Gaussian, the reparameterization trick \citep{Kingma14,Rezende14} is often employed for the gradient estimation.
With this technique, we can adjust a sampled configuration continuously within the domain space of the variables, where the variance of the gradient estimate is kept low.
While it has been shown to be efficient in various applications \citep{Gregor15,Heese15,Maaloe16}, it cannot be applied to discrete variables, since any reparameterization includes discontinuous operations for which the gradient cannot be estimated.
Instead, the likelihood-ratio method \citep{Glynn90,Williams92} is used for discrete variables, dispite its high-variance estimation.

In this study, we propose a simple way to apply reparameterization to discrete variables, while avoiding the discontinuity by marginalizing out the variable of interest.
This method is applicable to any kind of variable for which we can approximate the expectation directly, although we only consider the discrete case in this study.
The variance of the gradient estimate is guaranteed not to be larger than that of the likelihood-ratio method with the optimal input-dependent baseline.

Our algorithm requires us to marginalize out the discrete variable, for which we need to simulate all of its configurations.
The simulations are essential for the gradient estimation, because a simulation of any single configuration provides no information about the loss landscape over the other configurations.
Existing gradient estimators for discrete variables simulate each configuration separately, while in our algorithm, they are simulated all at once by sharing the reparameterized noise factors.
Our method can be viewed as an application of \emph{common random numbers} to these simulations, which is known to reduce the variance when the target value is expressed as a difference between two random variables.
It was applied to the finite-difference gradient estimator in \citet{Glynn89}, whereas, here, we apply it to the exact gradient computation of an expectation over a discrete variable.
It has greatly reduced variance compared to existing techniques.

\paragraph{Related work}
The likelihood-ratio method \citep{Glynn90,Williams92} is often used to make gradient estimations of discrete variables, in which only one configuration is evaluated by simulation at each iteration.
It requires multiple iterations to cover the information of all configurations, and it simulates different configurations separately.
Furthermore, the likelihood ratio becomes unstable when the probability mass concentrates in only a few configurations; this also causes high variance.
There are many techniques to reduce the variance \citep{Paisley12,Bengio13,Ranganath14,Mnih14,Gu16}, although their reductions are not enough for large and complex models.
In another approach, called local expectation gradient \citep{Titsias15}, the variable of interest is locally marginalized out with all other variables simulated only once.
Our method is deeply connected to this method; both behave equivalently if the variable has no descendants in the graphical model.
When the variable has descendants, the local expectation gradient simulates them only on one configuration of the variable, and thus requires multiple iterations to simulate all configurations, each of which is simulated separately.

The rest of this paper is organized as follows.
In \secref{releg}, we introduce our method and give a theoretical analysis of it.
We show experimental results in \secref{experiments} and give a brief conclution in \secref{conclusion}.

\section{Method} \label{sec:releg}

Our task is to estimate the gradient of $F(\phi; x) = \E_{\q} f(x, z)$, where $f$ is a feasible function, $\q = \prod_{i=1}^M \qi$ a directed model of an $M$-dimensional vector of variables $z = (z_1, \dots, z_M)$ conditioned on an input to the system $x$, and $\phi$ the model parameters.
Here $\pa_i$ denotes the parent nodes of $z_i$ in the graphical model.
For simplicity, we will assume that $\phi_i$ and $\phi_{i'}$ for $i \neq i'$ do not share any parameters, but this assumption can easily be removed.
One example of our setting is the gradient estimation for the variational inference of a generative model $p(x, z)$ with an approximate posterior $\q$, where the objective is given by $F$ with $f(x, z) = \log p(x, z) - \log \q$.
We will omit the gradient term corresponding to the dependency of $f$ on $\phi$ from our discussion, since it is easy to estimate.

Suppose each sample from a conditional $\qi$ is reparameterized as $z_i = \gi$, where $\noise_i$ is drawn from a noise distribution $p(\noise_i)$.
When $z_i$ is discrete, the gradient $\nabla_{\phi_i} F(\phi; x)$ cannot be estimated using the reparameterization trick, since $\gi$ is not continuous at some point of $\phi_i$.

We can bypass the discontinuity by marginalizing out $\noise_i$.
Here, let $\noise_{\except i}$ be the noise factors other than $\noise_i$.
We write the whole reparameterization as $z = \g$, and transform the gradient as follows.
\begin{equation}
  \nabla_{\phi_i} \E_{\q} f(x, z) = \nabla_{\phi_i} \E_{p(\noise)} f(x, \g) = \E_{p(\noise_{\except i})} \nabla_{\phi_i} \E_{p(\noise_i)} f(x, \g).
  \label{eq:releg-noise}
\end{equation}
This transformation comes from the observation that, even if $f(x, \g)$ is not continuous, its expectation over $p(\noise_i)$ is differentiable by $\phi_i$.
If this inner expectation can be computed, \eref{releg-noise} can be estimated by sampling $\noise_{\except i}$.

The inner expectation is computed as follows.
Let $z_{\except i}$ be the variables other than $z_i$ and $z_{\except i} = \gxi$ an ancestral sampling procedure of them with clamped $z_i$; i.e., $z_{i'}$ for each $i' \neq i$ is computed by $\gip$ with $z_i$ fixed to the given one.
The inner expectation is then transformed as $\E_{p(\noise_i)} f(x, \g) = \sum_{z_i} f(x, z) \qi$, with which we can rewrite \eref{releg-noise} as follows.
\begin{equation}
  \nabla_{\phi_i} \E_{\q} f(x, z)
  = \E_{p(\noise_{\except i})} \sum_{z_i} f(x, z) \nabla_{\phi_i} \qi \Big|_{z_{\except i} = \gxi}.
  \label{eq:releg}
\end{equation}
Note that the gradient $\nabla_{\phi_i} \qi$ can be computed analytically.
The simulated variables $z_{\except i}$ can contain discrete variables, which are left reparameterized with the discontinuous function $\gxi$.
The resulting algorithm is shown in \aref{releg}.

\begin{algorithm}
  \caption{
    Gradient estimation by \eref{releg}.
    Note that the procedure can be made more efficient by reusing variables that are not descendants of $z_i$ in the ancestral sampling at line \ref{state:simulation}.
  }
  \label{alg:releg}
  \begin{algorithmic}[1]
    \Require a set of parameters $\phi$ and an input variable $x$.
    \State Sample $\noise \sim p(\noise)$.
    \For{ $i = 1, \dots, M$ }
      \ForAll{ configurations of $z_i$ }
        \State $z_{\except i} := \gxi$.  \label{state:simulation}
        \State $f_{z_i} := f(x, z) \nabla_{\phi_i} \qi$.
      \EndFor
      \State $\Delta_i := \sum_{z_i} f_{z_i}$.
    \EndFor
    \State \Return $(\Delta_1, \dots, \Delta_M)$ as an estimation of $\nabla_\phi F(\phi; x)$.
  \end{algorithmic}
\end{algorithm}

For example, suppose that $z_i$ is a Bernoulli variable whose mean is given by $\mu_i = \mu_i(\pa_i, \phi_i)$.
For gradient estimations w.r.t. $\phi_{\except i}$, it can be reparameterized as $z_i = 1$ iff $\noise_i < \mu_i$ for $\noise_i \sim U(0, 1)$.
For the gradient estimation w.r.t. $\phi_i$, $\nabla_{\mu_i} \qi$ is $1$ if $z_i = 1$ and $-1$ otherwise; thus, the estimator is given by $(f_1 - f_0) \nabla_{\phi_i} \mu_i$, where $f_k$ denotes the value of $f(x, z)$ simulated with $z_i=k$ for $k \in \{0, 1\}$.
The variance of an estimation of $f_1 - f_0$ is given by $\V(f_1 - f_0) = \V f_0 + \V f_1 - 2\Cov(f_0, f_1)$, which is reduced by a large covariance of $f_0$ and $f_1$.
Our estimator reuses the same noise factor $\noise_{\except i}$ for simulations of these terms; thus, the covariance is expected to be large.
This technique is known as the method of common random numbers, with which our estimator enjoys low variance.

While the formulation is similar to the original reparameterization trick, one big difference is that the estimator \eqref{eq:releg} does not use the gradient of $f$ w.r.t. $z_i$.
This is essential in the case of $z_i$ being discrete, since the gradient of $f$ is not related to the expectation gradient in general.
This can be easily understood in the above Bernoulli case, where the exact expectation is written as a difference of $f$ on $z_i=1$ and $z_i=0$.
Even if $f$ is smoothly defined over $z_i \in [0, 1]$, there is no guarantee that the gradient of $f$ at a given $z_i$ approximates the true gradient, especially when $f$ is highly nonlinear.

\paragraph{Theoretical analysis}
The variance of our estimator is guaranteed not to be larger than that of the likelihood-ratio method.
Let $\phi_{ij}$ be the $j$-th element of the parameter vector $\phi_i$.
Here, we focus on the estimation of the partial derivative $\df{F}{\phi_{ij}}$.
The likelihood-ratio method can be formulated as a Monte-Carlo simulation of an expectation,
\begin{equation}
  \Df{F(\phi; x)}{\phi_{ij}}
  = \E_{\q} (f(x, z) - b) \Df{}{\phi_{ij}} \log \qi,  \notag
\end{equation}
  where $b$ is a \emph{baseline} that can depend on variables other than $z_i$ and its descendants.
Our estimator is a Monte Carlo simulation of \eref{releg} with $\nabla_{\phi_i}$ replaced by $\partial / \partial \phi_{ij}$.
Using these formulations, the following statement holds.
\begin{theorem} \label{thm:variance-lrb-releg}
  Let $\phi_{ij}$ be any parameter, $b$ any baseline, $\VLR(b)$ the variance of the likelihood-ratio estimator, and $\VReLEG$ that of the proposed estimator; then it holds that $\VReLEG \le \VLR(b)$.
\end{theorem}
In particular, our method always achieves a variance not larger than that of the likelihood-ratio method with the optimal input-dependent baseline.
The proof is given in the appendix.

\section{Experiments} \label{sec:experiments}

We empirically compared the likelihood-ratio method and our estimator in variational learning of sigmoid belief networks (SBNs) \citep{Neal92}.
For the deepest layer, the logit was directly parameterized in the generative model.
We used a reverse-directional SBN for the posterior approximation; i.e., the variational model infers latent variables from shallow layers to deep layers one by one.
The models we used were the same as those used in \citet{Mnih14}, except the number of layers and units.
We denote the architecture using a notation like SBN($H_L$-$\cdots$-$H_1$), where $H_\ell$ represents the number of units in the $\ell$-th layer.

We conducted experiments on the MNIST dataset \citep{LeCun98}, a set of 28x28 pixel gray-scale images of hand-written digits.
We binarized each image with the procedure described in \citet{Salakhutdinov08}.
We followed the standard data separation and used 10,000 images for testing and the rest for training.
We further divided the latter into 50,000 training images and 10,000 validation images.
We evaluated the model with the validation set at regular intervals throughout training and used the best model for the final test.

We trained the SBNs with RMSprop \citep{Tieleman12} using gradient estimates given by the likelihood-ratio method (LR) or our method (ours).
The learning rate was set to 0.001.
We used mini-batches of size 100 in all experiments, and applied a weight decay with a coefficient of 0.001 to all weight matrices (not to the bias parameters).
As for LR, we used the baseline proposed in \citet{Mnih14}, which consists of a running estimate of the expected loss and input-dependent loss estimation with layer-wise extra neural networks.
We did not use variance normalization \citep{Mnih14}, as RMSprop already achieves per-element variance normalization.

The computational cost of our method is $M$ times larger than that of the likelihood-ratio method, since simulations of $z_{\except i}$ for all $i = 1, \dots, M$ are required.
The cost is not as problematic as expected, since the additional factor of $M$ is easily parallelized.
In our experiments, using an NVIDIA GeForce TITAN X, the actual difference in computational times was less than two-fold with $400 \le M \le 800$.

\begin{figure}
  \begin{minipage}{0.5\hsize}
    \begin{center}
      \includegraphics[width=\hsize]{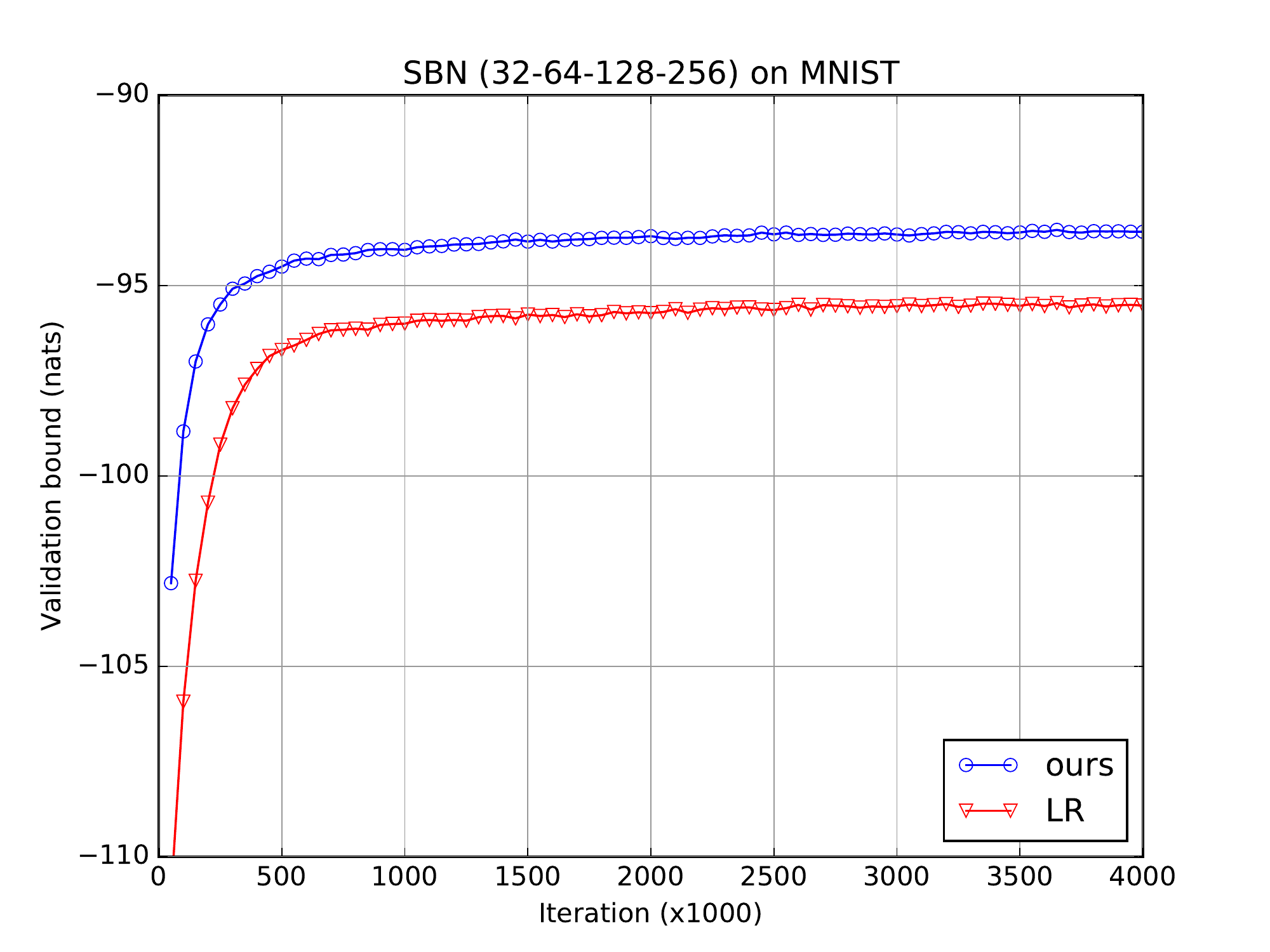}
    \end{center}
  \end{minipage}
  \begin{minipage}{0.5\hsize}
    \begin{center}
      \includegraphics[width=\hsize]{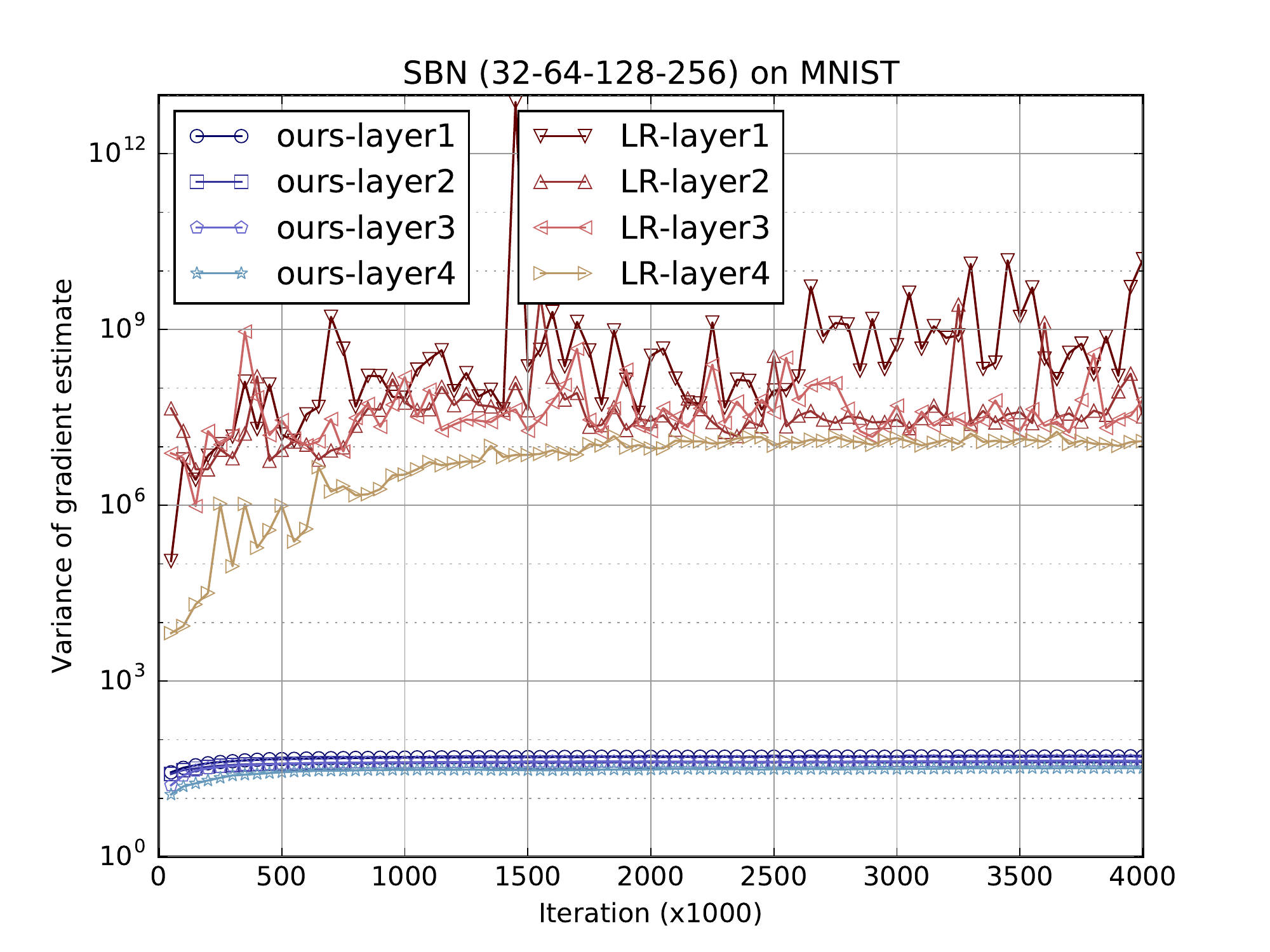}
    \end{center}
  \end{minipage}
  \caption{
    Left: variational lower bound of the log likelihood evaluated on the validation set (higher is better).
    Right: variance of the gradient w.r.t. mean parameters of Bernoulli variables for each layer.
    The variance of each unit is estimated using 50 million samples (i.e., 1,000 samples for each training image) and then averaged over all units in each layer.
  }
  \label{fig:training}
\end{figure}

\begin{table}
  \caption{Variational bound of the negative log likelihood on the test set with various architectures.}
  \label{tbl:vlb}
  \begin{center}
    \begin{tabular}{|r|r|r|r|r|} \hline
            & 200-200 & 200-200-200 & 200-200-200-200 & 32-64-128-256 \\ \hline
      LR    & 98.86 & 95.40 & 94.82 & 94.73 \\ \hline
      Ours & 98.28 & 95.03 & 93.67 & 92.79 \\ \hline
    \end{tabular}
  \end{center}
\end{table}

Figure \ref{fig:training} plots the validation performance and gradient variance for the four-layer model SBN(32-64-128-256).
It shows that our method learns much faster than LR with the input-dependent adaptive baseline.
The right figure shows the variance reduction effect.
The variance of our method is smaller than that of LR, where the difference ($\times 10^5$ to $\times 10^8$) is much larger than the difference in computational costs.
The model with the best validation score was used for evaluation on the test set, whose results are listed in Table \ref{tbl:vlb}.
For various architectures of SBNs, our method outperforms LR.
In particular, the difference is larger when the model is deeper.
This can be qualitatively understood by observing that the optimization of a deeper model becomes more difficult where the quality of the gradient estimate critically affects the optimization performance.

\section{Conclusion} \label{sec:conclusion}

We showed that reparameterization can still be applied to discrete variables, which enables us to use common random numbers in evaluations of multiple configurations.
The resulting method has lower variance; we confirmed this both theoretically and empirically.
Although its computational cost is worse than the existing methods, it empirically runs quickly enough; the additional cost can be alleviated by parallelizing the computation on modern GPUs.
Future work will include seeking a better way to balance the tradeoff between the computational cost and variance reduction.

\subsubsection*{Acknowledgments}

We thank Daisuke Okanohara for helpful discussions.

\section*{References}

\small
\begingroup
\renewcommand{\section}[2]{}  % Use section* above
\bibliography{releg}{}
\bibliographystyle{plainnat}
\endgroup

\appendix

\section{Proof of \theoref{variance-lrb-releg}}

We first introduce a well-known lemma used in our analysis.

\begin{lemma}[Variance partitioning] \label{lem:var-part}
  Let $X$ and $Y$ be sets of random variables, and $h = h(X, Y)$ a function on them.
  Then the following formula holds: $\V_{X, Y} h = \E_X \V_{Y | X} h + \V_X \E_{Y | X} h$.
\end{lemma}
\begin{proof}
  Note that $\V A = \E A^2 - (\E A)^2$ holds for any random variable $A$.
  Applying it to the right side of the formula yields
  \begin{align}
    \E_X \V_{Y | X} h + \V_X \E_{Y | X} h
    &= \E_X[\E_{Y | X} h^2 - (\E_{Y | X} h)^2] + \E_X(\E_{Y | X} h)^2 - (\E_X \E_{Y | X} h)^2  \notag \\
    &= \E_{X, Y} h^2 - (\E_{X, Y} h)^2 = \V_{X, Y} h.  \notag
  \end{align}
\end{proof}

\begin{proof}[Proof of \theoref{variance-lrb-releg}]
  Suppose that $z_{\except i}$ is reparameterized as is done in \secref{releg}.
  Denote the parent node of $z_i$ simulated with $\noise_{\except i}$ by $\pa_i(\noise_{\except i})$.
  The variance of the likelihood-ratio method is evaluated as follows.
  \begin{align}
    \VLR(b)
    &= \V_{\q} (f(x, z) - b) \Df{}{\phi_{ij}} \log \qi  \notag \\
    &= \V_{p(\noise_{\except i}) \qinoise} ( f(x, z_i, z_{\except i} = \gxi) - b) \Df{}{\phi_{ij}} \log \qinoise  \notag \\
    &= \E_{p(\noise_{\except i})} \V_{\qinoise} ( f(x, z_i, z_{\except i} = \gxi) - b) \Df{}{\phi_{ij}} \log \qinoise  \notag \\
    &+ \V_{p(\noise_{\except i})} \E_{\qinoise} ( f(x, z_i, z_{\except i} = \gxi) - b) \Df{}{\phi_{ij}} \log \qinoise.
    \label{eq:proof-intermid}
  \end{align}
    where we use Lemma \ref{lem:var-part} on $X = \noise_{\except i}$ and $Y = z_i$ in the last equation.
  Note that $\E_{\qi} b \Df{}{\phi_{ij}} \log \qi = 0$; thus, the second term of \eref{proof-intermid} can be further transformed as follows.
  \begin{align}
    &\V_{p(\noise_{\except i})} \E_{\qinoise} (f(x, z_i, z_{\except i} = \gxi) - b) \Df{}{\phi_{ij}} \log \qinoise  \notag \\
    &=\ \V_{p(\noise_{\except i})} \sum_{z_i} f(x, z_i, z_{\except i} = \gxi) \Df{}{\phi_{ij}} \qinoise  \notag \\
    &=\ \VReLEG.  \notag
  \end{align}
  Since the first term of \eref{proof-intermid}, which is an expectation of a variance, is not less than zero, we conclude that $\VLR(b) \ge \VReLEG$.
\end{proof}

\end{document}